\documentclass{article}

\usepackage{microtype}
\usepackage{graphicx}
\usepackage{subfigure}
\usepackage{booktabs} % for professional tables
\usepackage{array,booktabs,arydshln,xcolor}
\usepackage{graphicx,wrapfig,lipsum}
\usepackage{float}
\usepackage{natbib}
\usepackage[nocomma]{optidef}
\usepackage{eucal}
\usepackage{nomencl}

\usepackage{hyperref}
\usepackage{verbatim}
\usepackage{amssymb}
\usepackage{array}
\usepackage[ruled,linesnumbered,vlined]{algorithm2e}
\newcommand{\PreserveBackslash}[1]{\let\temp=\\#1\let\\=\temp}
\newcolumntype{C}[1]{>{\PreserveBackslash\centering}p{#1}}

\usepackage{tikz} 
\usetikzlibrary{automata, positioning, arrows}
\usepackage{comment}
\usepackage{amsmath}
\usepackage[noabbrev,capitalize]{cleveref}
\usepackage{bbm} 
\usepackage{dsfont}
\usepackage[preprint]{neurips_2021}
\allowdisplaybreaks

\usepackage[normalem]{ulem}
\usepackage{cases} %% seem to be needed to number array sub-equations
\newcommand{\tcr}[1]{\textcolor{red}{#1}}

\usepackage{nccmath}
\usepackage{mathtools}
\usepackage{nicefrac} 
\usepackage{xfrac} 

\usepackage{lipsum}

\usepackage{multicol}
\usepackage{lipsum}
\usepackage{tikz}
\usetikzlibrary{positioning,angles,quotes,arrows,automata}

\usepackage{float}
\usepackage{bm}
\usepackage{amsfonts}
\usepackage{amsbsy}
\usepackage{amsthm}
\usepackage{color}
\usepackage{amsmath}
\usepackage{enumerate}% http://ctan.org/pkg/enumerate
\usepackage{amssymb}
\usepackage{enumitem}
\usepackage{array}
\usepackage{nicefrac}

\usepackage{tabularx}
\usepackage{booktabs}
\usepackage{multirow}
\newcolumntype{Y}{>{\centering\arraybackslash}X}

\usepackage{color, colortbl}
\definecolor{Gray}{gray}{0.98}
\definecolor{LightCyan}{rgb}{0.88,1,1}
\newcolumntype{g}{>{\columncolor{Gray}}c}

\DeclareMathOperator*{\argmax}{argmax}
\DeclarePairedDelimiter\abs{\lvert}{\rvert}%
\DeclarePairedDelimiter\norm{\lVert}{\rVert}%
\makeatletter
\let\oldabs\abs
\def\abs{\@ifstar{\oldabs}{\oldabs*}}
\let\oldnorm\norm
\def\norm{\@ifstar{\oldnorm}{\oldnorm*}}
\makeatother

\newcommand{\eat}[1]{}
\newcommand{\states}{\mathcal{S}}
\newcommand{\actions}{\mathcal{A}}

\newcommand*{\tr}{^{\mkern-1.5mu\mathsf{T}}}

\newcommand{\Real}{\mathbb{R}}
\renewcommand{\ss}{\,:\,}

\newcommand{\simplexs}{\Delta ^{ \states }}

\newcommand{\ambset}{\mathcal{P}}

\newcommand{\dataset}{\mathcal{D}}
\newcommand{\opt}{^\star}

\usepackage{thmtools}
\usepackage{nameref}
\usepackage[noabbrev,capitalize]{cleveref}

\newtheorem{theorem}{Theorem}[section]

\newtheorem{remark}[theorem]{Remark}

\newtheorem{definition}{Definition}

\makeatletter
\newcommand{\myref}[1]{\cref{#1}\mynameref{#1}{\csname r@#1\endcsname}}
\newcommand{\Myref}[1]{\Cref{#1}\mynameref{#1}{\csname r@#1\endcsname}}

\def\mynameref#1#2{%
	\begingroup
	\edef\@mytxt{#2}%
	\edef\@mytst{\expandafter\@thirdoffive\@mytxt}%
	\ifx\@mytst\empty\else
	\space(\nameref{#1})\fi
	\endgroup
}
\makeatother

\DeclareMathOperator{\E}{\mathbb{E}}

\newcommand{\indicator}{\mathbbm{1}}
\newcommand{\lyapunov}{\mathbb{V}}
\newcommand{\rcmdp}{\mathfrak{M}}
\newcommand{\lyapolicy}{\mathcal{F}_\mathcal{L}}

\renewcommand{\ss}{\,:\,}
\newcommand{\statecount}{S}

\newcommand{\RBU}{\mathfrak{T}}
\newcommand{\LO}{\mathfrak{L}}
\newcommand{\policyparam}{{\pi_\theta}}

\newtheorem{proposition}{Proposition}
\title{Lyapunov Robust Constrained-MDPs: Soft-Constrained Robustly Stable Policy Optimization under Model Uncertainty}
\author{
Reazul Hasan Russel\hspace{16pt} Mouhacine Benosman \hspace{16pt} Jeroen Van Baar\hspace{16pt} Radu Corcodel\\ 
Mitsubishi
Electric Research Laboratories (MERL)\\ Cambridge, MA 02139, USA\\
\{{\tt rrussel}, {\tt benosman}, {\tt jeroen}, {corcodel}\}$@$ {\tt merl.com}
}

\date{}

\begin{document}

%\include{TODO}
%\newpage

\maketitle

\begin{abstract}
Safety and robustness are two desired properties for any reinforcement learning algorithm. CMDPs can handle additional safety constraints and RMDPs can perform well under model uncertainties. In this paper, we propose to unite these two frameworks resulting in robust constrained MDPs (RCMDPs). The motivation is to develop a framework that can satisfy safety constraints while also simultaneously offer robustness to model uncertainties. We develop the RCMDP objective, derive gradient update formula to optimize this objective and then propose policy gradient based algorithms. We also independently propose Lyapunov based reward shaping for RCMDPs, yielding better stability and convergence properties. \end{abstract}

\section{Introduction}
Reinforcement learning (RL) is a framework to address sequential decision-making problems~\citep{sutton2018reinforcement,szepesvari2010algorithms}. In RL, a decision maker learns a policy to optimize a long-term objective by interacting with the (unknown or partially known) environment. The RL agent obtains evaluative feedback usually known as reward or cost for its actions at each time step, allowing it to improve the performance of subsequent actions \citep{sutton2018reinforcement}. With the advent of deep learning, RL has witnessed huge successes in recent times~\citep{silver2017mastering}. However, since most of these methods rely on model-free RL, there are several unsolved challenges, which restrict the use of these algorithms for many safety critical physical systems~\citep{vamtutorial,Benosman2018}. For example, it is very difficult for most model-free RL algorithms to ensure basic properties like stability of solutions, robustness with respect to model uncertainties, etc. This has led to several research directions which study incorporating robustness, constraint satisfaction, and safe exploration during learning for safety critical applications. While robust constraint satisfaction and stability guarantees are highly desirable properties, they are also very challenging to incorporate in RL algorithms. The main goal of our work is to formulate this incorporation into robust constrained-MDPs (RCMDPs), and derive corresponding theories necessary to solve them.

Constrained Markov Decision Processes (CMDPs) are a super class of MDPs that incorporate expected cumulative cost constraints~\citep{Altman2004}. Several solution methods have been proposed in the literature for solving CMDPs: trust region based methods~\citep{Achiam2017CPO}, linear programming-based solutions~\citep{Altman2004}, surrogate-based methods{~\citep{CYA16,Dalal2018}}, Lagrangian methods~\citep{Geibel2005,Altman2004}. We refer to these CMDPs as {\it non-robust}, since they do not take model uncertainties into account. On the other hand, another line of work explicitly handles model uncertainties and is known as Robust MDPs (RMDPs)~\citep{Nilim2004,Wiesemann2013}. RMDPs consider a set of plausible models from so called ambiguity sets. They compute solutions that can perform well even for the worst possible realization of models~\citep{Russel2019beyond, Wiesemann2013, Iyengar2005}. However, unlike CMDPs, these RMDPs are not capable of handling safety constraints. 

Safety constraints are important in real-life applications~\citep{Altman2004}. One cannot afford to risk violating some given constraints in many real-life situations. For example, in autonomous cars, there are hard safety constraints on the car velocities and steering angles~\citep{Lin2018}. Moreover, training often occurs on a simulated environment for many practical applications. The goal is to mitigate the sample inefficiency of model-free RL algorithms~\citep{vanBaar2019may}. The result is then transferred to the real world, typically followed by fine-tuning, a process referred to as Sim2Real. The simulator is by definition inaccurate with respect to the targeted problem, due to  approximations and lack of system identification. Heuristic approaches like domain randomization \citep{vanBaar2019may} and meta-learning \citep{pmlr-v78-finn17a} try to address model uncertainty in this setting, but they often are not theoretically sound. In safety critical applications, it is expected that a trained policy in simulation will offer certain guarantees about safety, when transferred to the real-world.
 
 In light of these practical motivations, we propose to unite the two concepts of RMDPs and CMDPs, leading to a new framework we refer as RCMDPs. The motivation is to ensure both safety and robustness. The goal of RCMDPs is to learn policies that simultaneously satisfy certain safety constraints and also perform well under worst-case scenarios. %Indeed, robustness is equally (if not more) important in estimating the cumulative constraint costs along the whole trajectory in order to certify that there will be no unexpected violations, when the system is deployed in reality. That is, when deployed, the worst-case cumulative constrained-cost will not exceed a pre-determined safety budget.
The contributions of this paper are four-fold: 1) formulate the concept of RCMDPs and derive related theories, 2) propose gradient based methods to optimize the RCMDP objective, 3) independently derive a Lyapunov based reward shaping technique, and 4) empirically validate the utility of the proposed ideas on several problem domains.

The paper is organized as follows: \cref{sec:formulation} describes the formulation of our RCMDP framework and the objective we seek to optimize. A Lagrange-based approach is presented in \cref{sec:rcmdp_opt} along with required gradient update formulas and corresponding policy optimization algorithms. Section \ref{sec:L-rcmdp_opt} is dedicated to the Lyapunov stable RCMDPs and presents the idea of Lyapunov based reward shaping. We draw the concluding remarks in \cref{sec:conclusion}.

\section{Problem Formulation: RCMDP concept} \label{sec:formulation}
We consider Robust Markov Decision Processes (RMDPs) with a finite number of states $\states = \{1, \ldots, S \}$ and finite number of actions $\actions = \{1, \ldots, A\}$. Every action $a \in \actions$ is available for the decision maker to take in every state $s \in \states$. After taking an action $a\in\actions$ in state $s\in\states$, the decision maker transitions to a next state $s'\in\states$ according to the \emph{true}, but \emph{unknown}, transition probability $p_{s,a} \opt \in \simplexs$ and receives a reward $r_{s,a,s'} \in \Real$. We use $p_{s,a}$ to denote transition probabilities from $s\in\states$ and $a\in\actions$, and condense it to refer to a transition function as $p = \big( p_{s,a} \big)_{ s\in\states,a\in\actions} \in \big( \Delta^\states \big)^{\states\times\actions}$. We condense the rewards to vectors $r_{s,a} = \big( r_{s,a,s'} \big)_{s'\in\states}\in \Real^\states$ and $r = \big(r_{s,a}\big)_{s\in\states,a\in\actions}$.

Our RMDP setting assumes that the transition $p_{s,a}$ is chosen adversarially from an ambiguity set $\ambset_{s,a}\in \big(\Delta^\states\big)^{\states\times\actions}$ for each $s\in\states$ and $a\in\actions$. An ambiguity set $\ambset_{s,a}$, defined for each state $s\in\states$ and action $a\in\actions$, is a set of feasible transitions quantifying the uncertainty in transition probabilities. We restrict our attention to $s,a-$rectangular ambiguity sets which simply assumes independence between transition probabilities of different state-action pairs~\citep{LeTallec2007,Wiesemann2013}. We define the $L_1-$norm bounded ambiguity sets around the nominal transition probability $\bar{p}_{s,a}=\E[p\opt_{s,a}|\dataset]$, for some dataset \(\dataset\) as:
\begin{equation*}
\ambset_{s,a} = \bigl\{p \in \Delta^\statecount \ss \norm{p - \bar{p}_{s,a} }_1 \le \psi_{s,a} \bigr\},
\end{equation*}
where $\psi_{s,a}\ge 0$ is the budget of allowed deviations. This budget $\psi_{s,a}$ can be computed for each $s\in\states$, $a\in\actions$ using Hoeffding bound~\citep{petrik2019beyond}: $\psi_{s,a} = \sqrt{\frac{2}{n_{s,a}} \log \frac{S A 2^{S}}{\delta} }$, where $n_{s,a}$ is the number of transitions in dataset $\dataset$ originating from state $s$ and an action $a$, and $\delta$ is the confidence level. This $\psi_{s,a}$, if used to compute a policy in RMDPs, then guarantees that the computed return is a lower bound with probability $\delta$. Note, that this is just one specific choice for the ambiguity set. Our method can be extended to any other type of ambiguity set, e.g., $L_\infty-$norm, Bayesian, weighted, sampling based, etc. We use $\ambset$ to generally refer to $\ambset_\tau = \bigotimes_{ s_t\in\states, a_t\in\actions} \ambset_{s,a}$, where $\tau$ denotes the total number of time steps starting from $T-\tau$, with $T$ the length of the horizon, and $t\in\{T-\tau,T-\tau+1,\ldots,T\}$. For example, with $\tau=T$ we have $\ambset_T = \bigotimes_{s_t\in\states,a_t\in\actions}\ambset_{s,a}$ starting from time step $0$. This collectively represents the ambiguity set along with the notion of independence between state-action pairs in a tabular setting with discrete states and actions. Sampling-based sets under approximate methods, e.g., neural networks, for large and continuous problems also extend on this similar notion of ambiguity sets~\citep{Tamar2014,Derman2018}.

A stationary randomized policy $\pi(\cdot|s)$ for state $s\in\states$ defines a probability distribution over actions $a\in\actions$. The set of all randomized stationary policies is denoted by $\Pi\in \big(\Delta^\actions\big)^\states$. We parameterize the randomized policy for state $s\in\states$ as $\policyparam(\cdot|s)$ where $\theta \subseteq \Real^k$ is a $k-$dimensional parameter vector.  Let $\xi=\{s_0,a_0,c_0,d_0,\ldots,s_{T-1},a_{T-1},c_{T-1},d_{T-1},s_T\}$ be a sampled trajectory generated by executing a policy $\policyparam$ from a starting state $s_0\sim p_0$ under transition probabilities $p\in\ambset$, where $p_0$ is the distribution of initial states. Then the probability of sampling a trajectory $\xi$ is: $p^\policyparam(\xi)=p_0(s_0)\prod_{t=0}^{T-1}\policyparam(a_t|s_t)p(s_{t+1}|s_t,a_t)$ and the total reward along the trajectory $\xi$ is: $g(\xi, r) =  \sum_{t=0}^{T-1} \gamma^t r_{s_t,a_t,s_{t+1}}$~\citep{Puterman2005,sutton2018reinforcement}. The value function $v^\policyparam_p : \states\rightarrow\Real$ for a policy $\policyparam$ and transition probability $p$ is: $v^\policyparam_p = \E_{\xi\sim p}\big[ g(\xi, r) \big]$ and the total return is: \[\rho(\policyparam, p, r) = p_0^T v^\policyparam_p.\]

Because the RMDP setting considers different possible transition probabilities within the ambiguity set $\ambset$, we use a subscript $p$ (e.g. $v^\policyparam_p$) to indicate which one is used, in case it is not clear from the context.

We define a robust value function $\hat{v}_\ambset^\policyparam$ for an ambiguity set $\ambset$ as: $\hat{v}_\ambset^\policyparam = \min_{p\in\ambset}v^\policyparam_p$. Similar to ordinary MDPs, the robust value function can be computed using the robust Bellman operator~\citep{Iyengar2005,Nilim2005}:
\begin{equation*} \label{eq:bellman_definition}
\begin{aligned}
(\RBU_\ambset v)(s) := \max_{a\in\actions}\min_{p \in\ambset_{s,a}}  (r_{s,a} + \gamma \cdot p\tr v).
\end{aligned}
\end{equation*}
The optimal robust value function $\hat{v}\opt$, and the robust value function $\hat{v}_\ambset^\policyparam$ for a policy $\policyparam$ are unique and satisfy $\hat{v}\opt = \RBU_\ambset \hat{v}\opt$ and  $\hat{v}_\ambset^\policyparam = \RBU_\ambset^\policyparam \hat{v}^\policyparam$ ~\citep{Iyengar2005}. The robust return $\hat{\rho}(\policyparam,\ambset, r)$ for a policy $\policyparam$ and ambiguity set $\ambset$ is defined as~\citep{Nilim2005,Russel2019beyond}: 
\[\hat{\rho}(\policyparam,\ambset, r) = \min_{p\in\ambset} \rho(\policyparam, p, r) = p_0^T \hat{v}_\ambset^\policyparam,\]
where $p_0$ is the initial state distribution.

\paragraph{Constrained RMDP (RCMDP)} In addition to rewards $r_{s,a}$ for RMDPs described above, we incorporate a constraint cost $d_{s,a,s'}'\in \Real$, where $s,s'\in\states$ and $a\in\actions$, representing some kind of constraint on safety for the agent's behavior. Consider for example an autonomous car that makes money (reward $r$) for each complete trip but incurs a big fine (constraint cost $d$) for traffic violations or a collision. We define the constraint cost $d_{s,a,s'}'$ to be a negative reward $d_{s,a,s'} = -d_{s,a,s'}'$, which brings consistency in representing the \emph{worst-case} with a minimum over the ambiguity set $\ambset$ for both the objective and the constraint. An associated constraint budget $\beta\in\Real_+$ describes the total budget for constraint violations. This arrangement resembles the  constrained-MDP setting as described in \citep{Altman2004}, but with additional robustness. %We refer it as RCMDP and denoted with  $\rcmdp$.

Similar to reward based estimates described above, the total constraint cost along a trajectory $\xi$ is: $g(\xi, d) =  \sum_{t=0}^{\infty} \gamma^td_{s_t,a_t,s_{t+1}}$, the robust value function for policy $\policyparam$ and ambiguity set $\ambset$ is: $\hat{u}^\policyparam = \min_{p\in\ambset}\E_{\xi\sim p}\big[ g(\xi, d) \big]$ and the robust return: 
\[\hat{\rho}(\policyparam,\ambset, d) = \min_{p\in\ambset} \rho(\policyparam, p, d) = p_0^T \hat{u}^\policyparam.\]
Similar to $\hat{v}\opt$, the optimal constraint value function $\hat{u}\opt$ is also unique and independently satisfies the Bellman optimality equation~\citep{Altman2004}. We now formally define the objective of Robust Constrained MDP (RCMDP) as below:
%\begin{comment}
%\paragraph{Objective} Our objective is then to solve the %RCMDP optimization problem below:
%\begin{equation} \label{eq:rcmdp_objective}
%\begin{aligned}
%&\max_{\pi\in\Pi} \hat{\rho}(\pi,\ambset,r)\\
%&\text{s.t. } \hat{\rho}(\pi,\ambset,d) \ge \beta
%\end{aligned}
%\end{equation}
%\end{comment}
\begin{maxi!}{\policyparam\in\Pi}
	{\hat{\rho}(\policyparam,\ambset,r) \label{eq:rcmdp_obj},}
	{\label{eq:rcmdp_opt}}
	{} % optimization result
	\addConstraint{\hat{\rho}(\policyparam,\ambset,d)}{\ge \beta. \label{eq:rcmdp_constr}}
\end{maxi!}
This objective resembles the objective of a CMDP~\citep{Altman2004}, but with additional robustness integrated by the quantification of the uncertainty about the model. The interpretation of the objective is to find a policy $\policyparam$ that maximizes the worst-case return estimates, while satisfying the constraints in all possible situations.

%even for a fixed $\pi$, see Theorem 6 of for more details. \tcr{Is it true? The cited paper has two distinct MDPs induced by policies $\pi$ and $\pi_B$. But here if we fix policy $\pi$, we have one single MDP induced by $\pi$, the objective of this MDP is similar to (\ref{eq:crmdp_lagrange}) and actions represent realizations of transition uncertainties. And this resembles to solving a regular MDP.}

\section{Robust Constrained Optimization} \label{sec:rcmdp_opt}
A standard approach for solving the optimization problem (\ref{eq:rcmdp_opt}) is to apply the Lagrange relaxation procedure (\cite{Bertsekas2003}, Ch.3), which turns it into an unconstrained optimization problem:
\begin{equation} \label{eq:rcmdp_lagrange}
\begin{aligned}
%\max_{\policyparam\in\Pi} \min_{\lambda\in\Real_+}
\LO(\policyparam,\lambda) &= \hat{\rho}(\policyparam,\ambset,r) - \lambda \Big( \beta - \hat{\rho}(\policyparam,\ambset,d) \Big),
\end{aligned}
\end{equation}
where $\lambda$ is known as the \emph{Lagrange multiplier}. Note that, the objective in (\ref{eq:rcmdp_lagrange}) is non-convex and therefore is not tractable.
% in the primal domain. 
The dual function of $\LO(\policyparam,\lambda)$ involves a point-wise maximum with respect to $\policyparam$ and is written as~\citep{Paternain2019}:
\begin{equation*}
	d(\lambda) = \max_{\policyparam\in\Pi} \LO(\policyparam,\lambda).
\end{equation*}
The dual function $d(\lambda)$ provides an upper bound on (\ref{eq:rcmdp_lagrange}) and therefore needs to be minimized to contract the gap from optimality:
\begin{equation} \label{eq:rcmdp_lagrange_dual}
	\mathfrak{D}^\star = \min_{\lambda\in\Real_+} d(\lambda).
\end{equation}
The dual problem in (\ref{eq:rcmdp_lagrange_dual}) is convex and tractable, but the question remains about how large the duality gap is. In other words, how sub-optimal the solution $\mathfrak{D}^\star$ of the dual problem (\ref{eq:rcmdp_lagrange_dual}) is with respect to the solution of the original problem stated in (\ref{eq:rcmdp_opt}). To answer that question, \cite{Paternain2019} show that strong duality holds in this case under some mild conditions and the duality gap is arbitrarily small even with the parameterization ($\policyparam$) of policies. We thus aim to optimize the dual version of this problem using gradients.

\begin{proposition} \label{prop:rcmdp_obj_simple}
    The relaxed RCMDP objective of (\ref{eq:rcmdp_lagrange}) can be restated as:
    \begin{equation} \label{eq:rcmdp_lo_customized_obj}
        \LO(\policyparam,\lambda) = \sum_{\xi\in\Xi} p^\policyparam(\xi) \Big( g(\xi, r) + \lambda g(\xi, d) \Big) - \lambda \beta.
    \end{equation}
    \begin{proof}
    We defer the detailed derivation to Appendix \ref{prop:proof_rcmdp_obj_simple}.
    \end{proof}
\end{proposition}

The goal is then to find a saddle point $(\policyparam^*,\lambda^*)$ of $\LO$ in (\ref{eq:rcmdp_lo_customized_obj}) that satisfies $\LO(\policyparam,\lambda^*) \le \LO(\policyparam^*,\lambda^*) \le \LO(\policyparam^*,\lambda)$, $\forall \theta \in \Real^k$ and $\forall \lambda \in \Real_+$. This is achieved by ascending in $\theta$ and descending in $\lambda$ using the gradients of objective $\LO$ with respect to $\theta$ and $\lambda$ respectively~\citep{Chow2014}.
\begin{theorem} \label{th:lagrange_gradient_rule}
	The gradient of $\LO$ with respect to $\theta$ and $\lambda$ can be computed as:
	\begin{equation*}
	\begin{aligned}
	&\nabla_\theta \LO(\policyparam,\lambda) = \sum_\xi \hat{p}^\policyparam(\xi) \Big(g(\xi,r) + \lambda g(\xi,d) \Big) \sum_{t=0}^{T-1} \frac{\nabla_\theta \pi_\theta(a_t|s_t)}{\pi_\theta(a_t|s_t)},\\
	&\nabla_\lambda \LO(\policyparam,\lambda) = \sum_\xi \hat{p}^\policyparam(\xi) g(\xi,d) - \beta.
	\end{aligned}
	\end{equation*}
	
	\begin{proof} See Appendix \ref{apx:gradient_rule} for the detailed derivation.
	\end{proof}
\end{theorem}
With a fixed Lagrange multiplier $\lambda$, the constraint budget $\beta$ in (\ref{eq:rcmdp_lo_customized_obj}) offsets the sum by a constant amount. We can therefore omit this constant and define the Bellman operator for RCMDPs. We then show that this operator is a contraction.
%value function $\hat{w}^\pi_\ambset:\states\rightarrow\Real$ for RCMDPs as:
\begin{proposition} (Bellman Equation) \label{prop:rcmdp_bellman}
For a fixed policy $\policyparam$ and discount factor $\gamma$, the RCMDP value function $\hat{w}^\policyparam$ satisfies a Bellman equation for each $s\in\states$:
\begin{equation} \label{eq:rcmdp_bellman}
\hat{w}^\policyparam(s) = \min_{p\in\ambset_{s,\policyparam(s)}} \E_{s'\sim p}\Big[ r'_{s,\policyparam(s),s'} + \gamma \hat{w}^\policyparam(s') \Big],
\end{equation}
where \hspace{1pt} $r'_{s,\policyparam(s),s'} = r_{s,\policyparam(s),s'} + \lambda d_{s,\policyparam(s),s'}$.
\begin{proof}
The proof is deferred to Appendix \ref{ax:proof_rcmdp_bellman}.
\end{proof}
\end{proposition}
We define the Bellman optimality equation for RCMDPs as:
\begin{equation} \label{eq:rcmdp_bellman_optimality}
\begin{aligned}
(\RBU_\ambset^{\textit{rc}} \hat{w})(s) := \max_{a\in\actions}\min_{p \in\ambset_{s,a}}  (r'_{s,a} + \gamma \dot p\tr \hat{w}).
\end{aligned}
\end{equation}

\begin{proposition} (Contraction)
	The Bellman operator $\RBU_\ambset^{\textit{rc}}$ defined in (\ref{eq:rcmdp_bellman_optimality}) is a contraction.
	\begin{proof}
		The proof follows directly from Theorem 3.2 of~\cite{Iyengar2005}.
	\end{proof}
\end{proposition}
The RCMDP Bellman operator $\RBU_\ambset^{\textit{rc}}$ therefore satisfies the Bellman optimality equation and converges to a fixed point of the optimal RCMDP value function $\hat{w}\opt$.

\paragraph{Policy Gradient Algorithm} \label{subsec:pg_algo}

Algorithm \ref{alg:rcpg} presents a robust constrained policy gradient algorithm based on the gradient update rules derived above in Theorem \ref{th:lagrange_gradient_rule}. The algorithm proceeds in an episodic way based on trajectories and updates parameters based on the Monte-Carlo estimates. The algorithm requires an ambiguity set $\ambset$ as its input, which can be constructed with empirical estimates for smaller problems~\citep{Wiesemann2013,Russel2019beyond, Behzadian2021}. For larger problems it can be a parameterized estimate instead~\citep{Janner2019}.

\begin{algorithm*} [!h]
	\KwIn{A differentiable policy parameterization $\pi^\theta$, ambiguity set $\ambset$, confidence level $\alpha$, step size schedules $\zeta_2$ and $\zeta_1$.}
	\KwOut{Policy parameters $\theta$}
	%Initialize actor parameters $\theta\gets\theta_0$, and critic parameter $\lambda\gets \lambda_0$
	Initialize policy parameter: $\theta\gets\theta_0$
	
	\For{$k\gets0,1,2,\ldots$}{
		
		Sample initial state $s_0 \sim p_0$,
		initialize trajectory: $\xi \gets \emptyset$
		
		%\tcc{Simulate trajectory} %for $r$ with current policy $\theta$
		\For{$t\gets0,1,2,\ldots,T$}{
			Sample action: $a_t\sim \policyparam(\cdot|s_t)$
			
			Worst-case transitions with confidence $\alpha$: 
			$\hat{p}^{\policyparam} \gets \arg\min_{p\in\ambset_{s,a}} p^T\hat{v}^\policyparam$
			
			Sample next state: $s_{t+1}\sim\hat{p}^\policyparam$, observe $r_{s_t,a_t,s_{t+1}}$ and $d_{s_t,a_t,s_{t+1}}$.
			
			Record transition: $\xi \gets \Big\{s_t,a_t,s_{t+1},r_{s_t,a_t,s_{t+1}}, d_{s_t,a_t,s_{t+1}}, \frac{\nabla_\theta \policyparam(a_t | s_t)}{\policyparam(a_t | s_t)}\Big\}$
		}
		
		%\tcc{Loop backward and update parameters with $\xi$}
		$\theta$-update: $\theta \gets \theta + \zeta_2(k) \nabla_\theta \LO(\policyparam,\lambda)$ %\tcp*{$\theta$ update}
		
		$\lambda$-update: $\lambda \gets \lambda -  \zeta_1(k)\nabla_\lambda \LO(\policyparam,\lambda)$
	}
	
	\Return $\theta$;
	\caption{Robust-Constrained Policy Gradient (RCPG) Algorithm}    \label{alg:rcpg}
\end{algorithm*}

The step size schedules used in Algorithm \ref{alg:rcpg} satisfy the standard conditions for stochastic approximation algorithms~\citep{Borkar2009}. That is, $\theta$-update is on the fastest time-scale $\zeta_2(k)$, whereas $\lambda$-update is on a slower time-scale $\zeta_1(k)$, and thus results in a two time-scale stochastic approximation algorithm. We derive its convergence to a saddle point as below.

\begin{theorem} \label{th:rcmdp_pg_conv}
	Under assumptions \textbf{(A1)} - \textbf{(A7)} as stated in Appendix \ref{sec:ax_conv_alg}, the sequence of parameter updates of Algorithm \ref{alg:rcpg} converges almost surely to a locally optimal policy $\policyparam\opt$ as the number of trajectories $k \rightarrow \infty$.
	\begin{proof} We report the proof in Appendix \ref{subsec:apx_pg_convergence}.
	\end{proof}
\end{theorem}

\paragraph{Actor Critic Algorithm} \label{subsec:rcmdp_AC}

The general issue of having high variance in the Monte Carlo based policy gradient algorithm can be handled by introducing state values to use as baselines~\citep{sutton2018reinforcement}. As the optimal value function for RCMDPs can be computed using Bellman style recursive updates as shown in (\ref{eq:rcmdp_bellman}), an extension of the above PG algorithm to the actor-critic framework is straightforward. Algorithm (\ref{alg:rcmdp_AC}) reported in Appendix \ref{ax:rcmdp_ac_algorithm} presents an  actor critic (AC) algorithm for RCMDPs. The state-value parameterization with $f$ brings a new dimension in algorithm (\ref{alg:rcmdp_AC}) and results in a three time-scale stochastic algorithm. The convergence properties for this AC algorithm can be derived in a way similar to Theorem \ref{th:rcmdp_pg_conv} and we therefore omit the detailed derivations.

\section{Stable Robust-Constrained RL: Lyapunov-based RCMDP Concept}\label{sec:L-rcmdp_opt}
In this section, we propose Lyapunov-based\footnote{Other works have applied different notions of Lyapunov stability in the context of model-based RL \cite{FB2017,ETH_Bast ards2017} and MDPs \cite{Perkins2000,Chow2018Lyapunov}, however, none of these works incorporate explicit robustness in their formulation, i.e., in the context of RCMDP.} reward shaping for RCMDPs. The motivation of this is threefold: i) learn a good policy faster, ii) serve as a proxy to guide robustness when an estimate for the value function is not readily available and iii) guarantee stability (in the sense of Lyapunov) in the learning process. We first briefly introduce the idea of Lyapunov stability, Lyapunov function, and some of its useful characteristics. We then introduce the notion of additive shaping reward strategy based on Lyapunov functions and analyze its properties.

%A Lyapunov function qualitatively describes certain properties of a dynamical system to study the stability of the system evolving over time~\citep{Khalil2002}. A Lyapunov function can be interpreted as a scalar value representing the \emph{energy} of the system states. This energy descends monotonically along system trajectories and reaches to $0$ at a goal state of the system~\citep{Perkins2003, Perkins2014}. For example, an error or a distance measure of a state from a goal or desired state can be considered as a simple Lyapunov function. It is, however, not always straightforward to find a Lyapunov function for a dynamical system. 
\begin{definition} (Lyapunov stability)~\citep{Haddad2008}
Consider the general nonlinear discrete system $(Sy)\;s_{t+1}=f(s_{t})$, where $s\in D\in\mathbbm{R}^{n}$, $D$ is an open set containing $s^{\star}$, $f:\;D\rightarrow D$ is a continuous function on $D$. Then, the equilibrium point $s^{\star}$ of $(Sy)$ satisfying $s^{\star}=f(s^{\star})$, is said to be:

- Lyapunov stable if $\; \forall \epsilon>0$, $\exists \gamma(\epsilon)>0$, s.t., if $\|s_0-s^{\star} \|<\gamma$, then $\|s_t -s^{\star}\|<\epsilon,\;\forall t\in\mathbbm{Z}_{+}$ 

- Asymptotically stable if Lyapunov stable and $\exists\gamma>0$, s.t., if $\|s_0-s^{\star}\|<\gamma$, then $\lim\limits_{t\to \infty}\|s_t-s^{\star}\|= 0.$    
\end{definition}

\begin{definition} (Lyapunov direct method)~\citep{Haddad2008}
Consider the system (Sy), and assume that there exists a continuous Lyapunov function $\lyapunov:\;D\to\mathbbm{R}$, s.t.,
%\begin{equation}
%\begin{array}{l}
%    \lyapunov(s^{\star})=0\\
%    \lyapunov(s)>0,\;s\in\math{D}\setminus %\{s^{\star}\}\\
%    \lyapunov(f(s))-\lyapunov(s)\leq %0,\;s\in\math{D}\label{descent-prop},
%\end{array}
%\end{equation}
\begin{subnumcases}{}
 \lyapunov(s^{\star})=0\\
    \lyapunov(s)>0,\;s\in D\setminus \{s^{\star}\}\\
    \lyapunov(f(s))-\lyapunov(s)\leq 0,\;s\in D\label{descent-prop},
\end{subnumcases}
then the equilibrium point $s^{\star}$ is Lyapunov stable. If, in addition $ \lyapunov(f(s))-\lyapunov(s)<0,\;s\in D\setminus\{s^{\star}\}$, then $s^{\star}$ is asymptotically stable. 
\end{definition}

\subsection{Stability Constraints for RMDPs}\label{LS-MDPs}
We propose to incorporate the Lyapunov stability descent property (\ref{descent-prop}) as a constraint in the RCMDP objective (\ref{eq:rcmdp_opt}) %(\ref{eq:rcmdp_obj})-(\ref{eq:rcmdp_constr})
, where the constraint cost is given by $d\equiv d_{s}=-(\lyapunov(s_{t+1})-\lyapunov(s_t))$. We set the budget $\beta=0$ to enforce Lyapunov stability or set $\beta>0$ for achieving asymptotic stability. Note that in this setting, we assume that the only constraint cost is the stability cost $d_s$, and thus we are in the setting of RMPDs to which we add a virtual stability constraint cost. 
%Sampling $s_{t+1}$ from the worst-case transition probability $\hat{p}_{s,a} = \arg\min_{p\in\ambset_{s,a}} p^T v$ provides a robust estimate for $\lyapunov(s_{t+1})-\lyapunov(s_{t})$. 
In this setting, we apply Algorithm 1 to propose a Lyapunov stable-RCPG algorithm, and use the results of Theorem \ref{th:rcmdp_pg_conv}, to deduce its asymptotic convergence to a local optimal stationary policy for the infinite horizon case. We summarize this in the following proposition.
\begin{proposition}
Under assumptions \textbf{(A1)} - \textbf{(A7)} as stated in \cref{sec:ax_conv_alg}, the sequence of parameter updates of \cref{alg:rcpg}, where $d\equiv d_s,\;\beta=0,$ converges almost surely to a locally optimal a.s. Lyapunov stable policy $\theta\opt$ as $k \rightarrow \infty$. Furthermore, if $\beta>0$, the policy is a.s. asymptotically stable. 
\end{proposition}
\begin{proof}
Consider the control problem defined by (\ref{eq:rcmdp_opt}), under assumptions \textbf{(A1)} - \textbf{(A7)}, and where $d\equiv d_s=-(\lyapunov(s_{t+1})-\lyapunov(s_t))$. Then, based on Theorem \ref{th:rcmdp_pg_conv}, we can conclude that Algorithm 1, converges asymptotically almost surely to a local optimal policy $\theta\opt$. Furthermore, since $\theta\opt$ is computed under the constraint of Lyapunov descent property in expectation, the equilibrium point of the controlled system is a.s.\footnote{Almost surely--a.s.--(asymptotic) Lyapunov stability is to be understood as (asymptotic) Lyapunov stability for almost all samples of the states.} Lyapunov stable (Definition 3.5, \cite{MJZ2003}) when $\beta=0$, and a.s. asymptotically Lyapunov stable (Definition 3.8, \cite{MJZ2003}) when $\beta>0$.
\end{proof}
\subsection{Stability Constraints for RCMDPs}
In the case where the problem at hand is an RCMPD with a constraint cost $d$ (e.g. physical obstacle avoidance constraints for a mobile robot), we propose two main approaches to incorporate the stability descent constraint. We take the parallel between the notions of soft constraints, where the Lyapunov descent constraints is not enforced as a constraint cost as in Sec. \ref{LS-MDPs}, and reward shaping~\citep{Ng1999}. Indeed, we propose to add the Lyapunov stability descent constraint directly to the reward $r$ of the RCMDP  (\ref{eq:rcmdp_opt}). 
\paragraph{Reward Shaping with Lyapunov Constraint} We define the shaping reward function $f_{s,a,s'}\rightarrow\Real$ based on this Lyapunov descent property.
\begin{equation}\label{eq:shape_reward}
    f_{s,a,s'} = -(\lyapunov(s')-\lyapunov(s))    
\end{equation}
The motivation behind this is quite intuitive: a transition towards descend direction leads to a desired region of the state space faster and therefore should be rewarded. So, if we were to receive a reward $r_{s,a,s'}$ in the original setting, we instead would pretend to receive a reward of $r_{s,a,s'}+f_{s,a,s'}$ on the same event. This renders a transformed RCMDP $\rcmdp'$ with same state space, action space and transition probabilities. Only the reward function is reshaped with additional reward signals $f$. %This reward transformation can be readily integrated into (\ref{eq:rcmdp_obj}) and/or (\ref{eq:rcmdp_constr}) to reshape the reward and/or constraint cost in our original RCMDP objective defined in (\ref{eq:rcmdp_opt}) without muddling anything else.
\begin{theorem} \label{th:reward_shaping_policy}
    Every optimal finite-horizon policy in transformed RCMDP $\rcmdp'$ is also an optimal finite-horizon policy in the original RCMDP $\rcmdp$ under Lyapunov based reward transformation stated in (\ref{eq:shape_reward}). Furthermore, under the assumption of transient MDP, every infinite-horizon policy in transformed RCMDP $\rcmdp'$ is also an optimal finite-horizon policy in the original RCMDP $\rcmdp$.
   \end{theorem} 
    \begin{proof}
    In the finite-horizon case, this result is a simple extension of Theorem 1 of~\cite{Ng1999} into the RCMDP setting and  the proof follows directly from~\cite{Ng1999}. In the infinite-horizon case, one needs to rely on the transient assumption for the MDP (in the sense of Def. 7.1 in \cite{Altman2004}) to conclude about the convergence of the finite-horizon problem to the infinite-horizon problem, using the arguments in (Theorem 15.1, \cite{Altman2004}) .
    See \cref{subsec:apx_reward_shaping_policy_invariance_proof} for the full derivation.
    \end{proof}
    \begin{remark}
    Note that the concept of Lyapunov reward transformation is independent of the RL algorithm, and thus can be applied with any existing mainstream approaches such as TRPO, PPO, or CPO. The Lyapunov reward transformation will allow faster convergence for these existing approaches, as verified in our empirical analysis. 
    \end{remark}

\section{Conclusion} \label{sec:conclusion}
In this paper, we studied robust constrained MDPs (RCMDPs) to simultaneously deal with constraints and model uncertainties in reinforcement learning. We proposed the RCMDP framework, derived related theoretical analysis and proposed algorithms to optimize the objective of RCMDPs. We also proposed an extension to Lyapunov-RCMDPs (L-RCMDPs) for RCMDPs based on the Lyapunov function. We analyzed the performance of our L-RCMDP algorithms in the context of reward-shaping. We provided theoretical analysis of Lyapunov stability and asymptotic convergence for our methods. We also empirically validated the proposed algorithms on three different problem domains. Future work should focus on automated learning of the Lyapunov function from the domain itself and apply the proposed approach to more complex practical problem domains. %It also remains to bring this
%We proposed to merge together the concepts of constrained MDPs and robust MDPs, leading to the concept of robust constrained MDPs (RCMDPs). Indeed, by doing so, one can take  of the safety guarantees given by the CMDP formulation, as well as the robustness guarantees w.r.t. model uncertainties, given by the RMDP formulation. We then proposed a robust soft-constrained Lagrange-based solution to the RCMDP problem, and a corresponding policy gradient algorithm. Next work will focus on extending the proposed approach to continuous domains, and validate the performance of this RCMDP formulation on more safety critical examples, e.g., robotics test-beds. 

\newpage
\bibliographystyle{icml2019}
%\bibliography{bellman}
\bibliography{arxiv_ver1}
\newpage 

\newpage
\appendix
\onecolumn

\section{RCMDP Derivations}

\subsection{Proof of Proposition \ref{prop:rcmdp_obj_simple}} \label{prop:proof_rcmdp_obj_simple}
We rewrite the objective (\ref{eq:rcmdp_lagrange}) and perform some algebraic manipulation as below:
\begin{equation*} \label{eq:crmdp_lagrange_extend}
\begin{aligned}
\LO(\policyparam,\lambda) &= \hat{\rho}(\policyparam,\ambset,r) - \lambda \Big( \beta - \hat{\rho}(\policyparam,\ambset,d) \Big)\\
&\stackrel{(a)}{=} \min_{p\in\ambset}\E_{\xi_1\sim p}\big[ g(\xi_1, r) \big] - \lambda \Big( \beta - \min_{q\in\ambset} \E_{\xi_2\sim q}\big[ g(\xi_2, d) \big]\Big)\\
&\stackrel{(b)}{=} \E_{\xi_1\sim\tilde{p}}\big[ g(\xi_1, r) \big] + \lambda \E_{\xi_2\sim\tilde{q}} \big[ g(\xi_2, d) \big] - \lambda \beta\\
&= \sum_{\xi_1\in\Xi_{\tilde{p}}} p^\policyparam(\xi_1) g(\xi_1, r) + \lambda \sum_{\xi_2\in\Xi_{\tilde{q}}} p^\policyparam(\xi_2) g(\xi_2, d) - \lambda \beta
\end{aligned}
\end{equation*}
Where $\Xi_{\tilde{p}}$ is the set of all possible trajectories induced by policy $\policyparam$ under transition function $\tilde{p}$. Similarly, $\Xi_{\tilde{q}}$ is the set of all possible trajectories induced by policy $\policyparam$ under transition function $\tilde{q}$. Step $(a)$ above follows by assuming that the initial state distribution $p_0$ concentrates all of its mass to one single state $s_0$. And $(b)$ follows with $\tilde{p} = \arg\min_{p\in\ambset}\E_{\xi_1\sim p}\big[ g(\xi_1,r) \big]$ and $\tilde{q} = \arg\min_{q\in\ambset}\E_{\xi_2\sim q}\big[ g(\xi_2,d) \big]$. Note that, $\tilde{p}$ and $\tilde{q}$ are distinct, independent and depend on rewards $r$ and constraint costs $d$ respectively. However, the rewards and constraint costs are coupled together in reality, meaning that the set of two trajectories $\Xi_{\tilde{p}}$ and $\Xi_{\tilde{q}}$ would not be different. So we select one set of trajectories $\Xi$ being either $\Xi_{\tilde{p}}$ or $\Xi_{\tilde{q}}$. This selection of $\Xi$ may happen based on our priorities toward robustness of reward $r$ (with corresponding trajectory $\Xi_{\tilde{p}}$) or constraint cost $d$ (with corresponding trajectory $\Xi_{\tilde{q}}$). Or, it can also be the best (e.g. yielding higher objective value) set among $\Xi_{\tilde{p}}$ and $\Xi_{\tilde{q}}$ satisfying the constraint. We then have a simplified formulation for $\LO$ as below:
\begin{equation*}
\LO(\policyparam,\lambda) = \sum_{\xi\in\Xi} p^\policyparam(\xi) \Big( g(\xi, r) + \lambda g(\xi, d) \Big) - \lambda \beta
\end{equation*}

\subsection{Proof of Theorem \ref{th:lagrange_gradient_rule}} \label{apx:gradient_rule}
\begin{proof}
	The objective as specified in (\ref{eq:rcmdp_lo_customized_obj}):
	\begin{equation*}
		\LO(\policyparam,\lambda) = \sum_{\xi\in\Xi} p^\policyparam(\xi) \Big( g(\xi, r) + \lambda g(\xi, d) \Big) - \lambda \beta
	\end{equation*}
	We first derive the gradient update rule of $\LO(\policyparam,\lambda)$ with respect to $\theta$ as below:
	\begin{equation*} \label{eq:grad_theta}
	\begin{aligned}
	\nabla_\theta \LO(\policyparam,\lambda) &= \sum_{\xi\in\Xi} \nabla_\theta p^\policyparam(\xi) \Big( g(\xi, r) + \lambda g(\xi, d)\Big)\\
	&= \sum_{\xi\in\Xi} p^\policyparam(\xi) \Big(g(\xi,r) + \lambda g(\xi,d) \Big) \nabla_\theta \log p^\policyparam(\xi) \\
	&= \sum_{\xi\in\Xi} p^\policyparam(\xi) \Big(g(\xi,r) + \lambda g(\xi,d) \Big) \nabla_\theta \log \bigg( p_0(s_0)\prod_{t=0}^{T-1} p(s_{t+1}|s_t,a_t) \pi_\theta(a_t|s_t) \bigg) \\
	&= \sum_{\xi\in\Xi} p^\policyparam(\xi) \Big(g(\xi,r) + \lambda g(\xi,d) \Big) \nabla_\theta  \bigg( \log p_0(s_0) + \sum_{t=0}^{T-1} \log p(s_{t+1}|s_t,a_t) + \log \pi_\theta(a_t|s_t) \bigg) \\
	&= \sum_{\xi\in\Xi} p^\policyparam(\xi) \Big(g(\xi,r) + \lambda g(\xi,d) \Big) \sum_{t=0}^{T-1} \nabla_\theta \log \pi_\theta(a_t|s_t) \\
	&= \sum_{\xi\in\Xi} p^\policyparam(\xi) \Big(g(\xi,r) + \lambda g(\xi,d) \Big) \sum_{t=0}^{T-1} \frac{\nabla_\theta \pi_\theta(a_t|s_t)}{\pi_\theta(a_t|s_t)}\\
	\end{aligned}
	\end{equation*}
	%Notice that the constraint budget $\beta$ does not play any role in the policy optimization.
	
	Next, we derive the gradient update rule for $\LO(\policyparam,\lambda)$ with respect to $\lambda$:
	\begin{equation*} \label{eq:grad_lambda}
	\begin{aligned}
	\nabla_\lambda \LO(\policyparam,\lambda) &= \nabla_\lambda \Bigg( \sum_{\xi\in\Xi} p^\policyparam(\xi) \Big( g(\xi, r) + \lambda g(\xi, d) \Big) - \lambda \beta \Bigg)\\
	&= \sum_{\xi\in\Xi} p^\policyparam(\xi) g(\xi,d) - \beta
	\end{aligned}
	\end{equation*}
\end{proof}

\subsection{Proof of Proposition \ref{prop:rcmdp_bellman}} \label{ax:proof_rcmdp_bellman}
\begin{proof}
\begin{equation*}
	\begin{aligned}
	\hat{w}^\policyparam(s) &= \min_{p\in\ambset_T} \E_{\xi\sim p}\big[ g(\xi, r) + \lambda g(\xi, d) \big]\\
	&\stackrel{(a)}{=} \min_{p\in\ambset_T} \E_{\xi\sim p} \Big[ r_{s,\policyparam(s),s'} + \gamma r_{s',\policyparam(s'),s''} + \gamma^2 r_{s'',\policyparam(s''),s'''} \ldots \\
	&\hspace{1.7cm} + \lambda \big( d_{s,\policyparam(s),s'} + \gamma d_{s',\policyparam(s'),s''} + \gamma^2 d_{s'',\policyparam(s''),s'''} + \ldots \big) | \xi \Big]\\
	&= \min_{p\in\ambset_T}\E_{\xi\sim p}\Big[ \big(r_{s,\policyparam(s),s'} + \lambda d_{s,\policyparam(s),s'}\big) + \gamma \big(r_{s',\policyparam(s'),s''} + \lambda d_{s',\policyparam(s'),s''} \big) \\
	&\hspace{1.7cm} + \gamma^2 \big(r_{s'',\policyparam(s''),s'''} + \lambda d_{s'',\policyparam(s''),s'''} \big) + \ldots | \xi \Big]\\
	&= \min_{p\in\ambset_T}\E_{\xi\sim p}\Big[ r'_{s,\policyparam(s),s'} + \gamma r'_{s',\policyparam(s'),s''} + \gamma^2 r'_{s'',\policyparam(s''),s'''} + \ldots | \xi \Big]\\
	&\stackrel{(b)}{=} \min_{p\in\ambset_{s,\policyparam(s)}}\E_{s'\sim p}\Big[ r'_{s,\policyparam(s),s'} + \gamma \min_{p\in\ambset_{T-1}}\E_{\xi'\sim p} \big[r'_{s',\policyparam(s'),s''} + \gamma r'_{s'',\policyparam(s''),s'''} + \ldots | \xi' \big] \Big]\\
	&= \min_{p\in\ambset_{s,\policyparam(s)}}\E_{s'\sim p}\Big[ r'_{s,\policyparam(s),s'} + \gamma \hat{w}^\policyparam(s') \Big]
	\end{aligned}
\end{equation*}
Here $(a)$ follows by expanding total return given a trajectory $\xi$ and $(b)$ follows by evaluating the one-step immediate transition apart. 
\end{proof}

\subsection{Actor-Critic Algorithm} \label{ax:rcmdp_ac_algorithm}
\begin{algorithm} [!h]
	\KwIn{A differentiable policy parameterization $\policyparam$, a differentiable state-value function $w^{\policyparam}(s, f)$, confidence level $\alpha$, step size schedule $\zeta_1$ and $\zeta_2$.}
	\KwOut{Policy parameters $\theta$}
	Initialize policy parameter $\theta\in \Real^k$ and state-value weights $f\in \Real^{k'}$\;
	\For{$j\gets0,1,2,\ldots$}{
		Sample initial state $s_0\sim p_0$, set time-step
		$t\gets 0$\;
		%\tcc{Loop for each step along a trajectory}
		\While{$s_t$ \text{is not terminal}}{

			Sample action: $a_t\sim \policyparam(\cdot|s_t)$

			Worst-case transitions with confidence $\alpha$: 
			$\hat{p}^{\policyparam} \gets \arg\min_{p\in\ambset_{s,a}} p^Tw^\policyparam$

			Sample next state $s_{t+1}\sim\hat{p}^\policyparam$ and observe $r_{s_t,a_t,s_{t+1}}$ and $d_{s_t,a_t,s_{t+1}}$;

			TD error: $\delta_t \gets r'_{s_t,a_t,s_{t+1}} + \gamma w^{\policyparam}(s_{t+1}, f) - w^{\policyparam}(s_{t}, f)$\;
			%\tcc{Update parameters with gradient estimates}
			\text{$\theta$ update:} $\theta \gets \theta + \zeta_2(k) \delta_t \nabla_{\theta} \LO( \policyparam, \lambda)$\;
			\text{$f$ update:} $f \gets f + \zeta_1(k) \delta_t \nabla_{f} w^{\policyparam}(s_{t}, f)$\;
			$t\gets t+1$\;
		}
	}
	\Return $\theta$ \;
	\caption{Robust Constrained  Actor Critic (RC-AC) Algorithm}    \label{alg:rcmdp_AC}
\end{algorithm}

\subsection{Convergence Analysis of Algorithm} \label{sec:ax_conv_alg}
\paragraph{Assumptions}

\textbf{(A1)} For any state $s$, policy $\policyparam(.|s)$ is continuously differentiable with respect to parameter $\theta$ and $\nabla_\theta \policyparam(.|s)$ is a Lipschitz function in $\theta$ for every $s\in\states$ and $a\in\actions$.

\textbf{(A2)} The step size schedules $\{ \zeta_2(t), \zeta_1(t) \}$ satisfy:

\begin{equation} \label{eq:rcmdp_step_sum}
\sum_t \zeta_1(t) = \sum_t \zeta_2(t) = \sum_t \zeta_3(t) = \infty
\end{equation}
\begin{equation} \label{eq:rcmdp_step_sum_sq}
\sum_t \zeta_1(t)^2, \sum_t \zeta_2(t)^2 \le \infty
\end{equation}
\begin{equation} \label{eq:rcmdp_step_order}
\zeta_1(t) = o\big(\zeta_2(t)\big)
\end{equation}

%states that the second moment of a martingale is bounded and therefore the martingale convergence theorem is applicable in the analysis
These assumptions are basically standard step-size conditions for stochastic approximation algorithms~\citep{Borkar2009}. Equation (\ref{eq:rcmdp_step_sum}) ensures that the discretization covers the entire time axis. (\ref{eq:rcmdp_step_sum_sq}) ensures that the errors resulting from the discretization of the Ordinary Differential Equation (ODE) and errors due to the noise %$\{M_n\}$
both becomes negligible asymptotically with probability one~\citep{Borkar2009}. Equations (\ref{eq:rcmdp_step_sum}) and (\ref{eq:rcmdp_step_sum_sq}) together ensure that the iterates asymptotically capture the behavior of the ODE. (\ref{eq:rcmdp_step_order}) mandates that, updates corresponding to $\zeta_1(t)$ are on a slower time scale than $\zeta_2(t)$.

\subsubsection{Policy Gradient Algorithm} \label{subsec:apx_pg_convergence}
The general stochastic approximation scheme used by \cite{Borkar2009} is of the form:

\begin{equation} \label{eq:rcmdp_general_iterate}
x_{t+1} = t_n + a(t)[h(x_t) + \Delta_{t+1}]
\end{equation}

where $\{\Delta_t\}$ are a sequence of integrable random variables representing the noise sequence and $\{a_t\}$ are step sizes (e.g. $\zeta(t)$). The expression $h(x_t)+\Delta_{t+1}$ inside the square bracket is the noisy measurement where $h(x_t)$ and $\Delta_{t+1}$ are not separately available, only their sum is available. The terms of (\ref{eq:rcmdp_general_iterate}) need to satisfy below additional assumptions:

\paragraph{(A3)} The function $h : \Real^d\rightarrow \Real^d$ is Lipschitz. That is $\lVert h(x) - h(y) \rVert \le L \lVert x-y \rVert$ for some $0\le L \le\infty$.

\paragraph{(A4)} $\{\Delta_t\}$ are martingale difference sequence:
\begin{equation*}
\E[\Delta_{t+1}|x_n,\Delta_n,n\le t] = 0
\end{equation*}

In addition to that, $\{\Delta_t\}$ are square-integrable:

\begin{equation*}
\E[\lVert \Delta_{t+1} \rVert^2 | x_n,\Delta_n,n\le t] \le K(1+\lVert x_t \rVert^2) \text{ a.s. for } t\ge 0,
\end{equation*}
and for some constant $K>0$.

Our proposed policy gradient algorithm is a two time-scale stochastic approximation algorithm. The parameter update iterations of the policy gradient algorithm are defined as below:

\begin{equation} \label{eq:rcmdp_up_it_theta}
\theta_{t+1} = \theta_t + \zeta_2(t) \nabla_\theta \LO(\policyparam,\lambda)
\end{equation}

\begin{equation} \label{eq:rcmdp_up_it_lambda}
\lambda_{t+1} = \lambda_t + \zeta_1(t)\nabla_\lambda \LO(\policyparam,\lambda)
\end{equation}

These gradient update rules defined in (\ref{eq:rcmdp_up_it_theta}) and (\ref{eq:rcmdp_up_it_lambda}) are in a special form as:

\begin{equation} \label{eq:rcmdp_special_iterate}
x_{t+1} = x_t + a(t)f(x_t,\epsilon_t), t\ge 0
\end{equation}

Where $\{\epsilon\}$ is a zero mean i.i.d. random variable representing noise.
%and the map $f:\Real^d\times\Real^d \rightarrow \Real^d $. 
To apply general convergence analysis techniques derived for (\ref{eq:rcmdp_general_iterate}) in \cite{Borkar2009}, we take the special form in (\ref{eq:rcmdp_special_iterate}) and transform it to the general format of (\ref{eq:rcmdp_general_iterate}) as below:

\begin{equation} \label{eq:rcmdp_transform}
h(x) = \E\big[ f(x,\epsilon_1) \big] \text{ and } \Delta_{n+1} = f(x_n,\epsilon_{n+1}) - h(x_n)
\end{equation}

With these transformation techniques, we obtain the general update for $\theta$ from (\ref{eq:rcmdp_up_it_theta}):

\paragraph{$\theta$ update:}
\begin{equation} \label{eq:rcmdp_gen_it_theta}
\theta_{t+1} = \theta_t + \zeta_2(t) \big[h(\theta_t,\lambda_t) + \Delta^{(1)}_{t+1}\big]
\end{equation}

where, $f^{(1)}(\theta_t,\lambda_t)=\nabla_\theta L(\policyparam,\lambda)$ is the gradient w.r.t $\theta$, $h(\theta_t,\lambda_t) = \E[f^{(1)}(\theta_t,\lambda_t)]$, and $\Delta^{(1)}_{t+1}=f^{(1)}(\theta_t,\lambda_t)-h(\theta_t,\lambda_t)$. Note that, the noise term $\epsilon$ is omitted because the noise is inherent in our sample based iterations.

\begin{proposition}
	$h(\theta_t,\lambda_t)$ is Lipschitz in $\theta$.
	
	\begin{proof}
		
		Recall that the gradient of $\LO(\policyparam,\lambda)$ with respect to $\theta$ is: 
		\begin{equation} \label{eq:ax_grad_theta}
		\begin{aligned}
		\nabla_\theta \LO(\policyparam,\lambda) = \sum_{\xi\in\Xi} p^\policyparam(\xi) \Big(g(\xi,r) + \lambda g(\xi,d) \Big) \sum_{t=0}^{T-1} \frac{\nabla_\theta \pi_\theta(a_t|s_t)}{\pi_\theta(a_t|s_t)}
		\end{aligned}
		\end{equation}
		
		Assumption (A1) implies that, $\nabla_\theta \pi_\theta(a_t | s_t)$ in the equation (\ref{eq:ax_grad_theta}) is a Lipschitz function in $\theta$ for any $s\in\states$ and $a\in\actions$. As the expectation of sum of $|T|$ number of Lipschitz functions is also Lipschitz, we conclude that $h(\theta_t,\lambda_t)$ is Lipschitz in $\theta$.
		
	\end{proof}
\end{proposition}

\begin{proposition}
	$\Delta^{(1)}_{t+1}$ of (\ref{eq:rcmdp_gen_it_theta}) satisfies assumption (A4).
\end{proposition}

We transform our update rule of (\ref{eq:rcmdp_up_it_lambda}) as:
\paragraph{$\lambda$ update:}

\begin{equation} \label{eq:rcmdp_gen_it_lambda}
\lambda_{t+1} = \lambda_t + \zeta_1(t) \big[g(\theta_t,\lambda_t) + \Delta^{(2)}_{t+1}\big]
\end{equation}

where, $f^{(2)}(\theta_t,\lambda_t)=\nabla_\lambda L(\policyparam,\lambda)$ is the gradient w.r.t $\lambda$, $g(\theta_t,\lambda_t) = \E_M[ f^{(2)} (\theta_t,\lambda_t)]$, and $\Delta^{(2)}_{t+1} = f^{(2)}(\theta_t,\lambda_t)-h(\theta_t,\lambda_t)$. 

Notice that $\nabla_\lambda \LO(\policyparam,\lambda) = \sum_\xi \hat{p}^\theta(\xi) g(\xi,d) - \beta$ is a constant function of $\lambda$. And therefore, $g(\theta_t,\lambda_t)$ is a constant function of $\lambda$.

\begin{proposition}
	$\Delta^{(2)}_{t+1}$ of (\ref{eq:rcmdp_gen_it_lambda}) satisfies assumption (A4).
\end{proposition}

We now focus on the singularly perturbed ODE obtained from (\ref{eq:rcmdp_gen_it_theta}) and (\ref{eq:rcmdp_gen_it_lambda}).

\begin{equation} \label{eq:rcmdp_theta_ode}
\dot{\theta} = \zeta_2(t) h(\theta_t,\lambda_t)
\end{equation}

\begin{equation} \label{eq:rcmdp_lambda_ode}
\dot{\lambda} = \zeta_1(t) g(\theta_t, \lambda_t)
\end{equation}

With assumption (A2), $\lambda(\cdot)$ is quasi-static from the perspective of $\theta(\cdot)$ turning (\ref{eq:rcmdp_theta_ode}) into an ODE. where $\lambda$ is held fixed:

\begin{equation} \label{eq:rcmdp_theta_ode_fixed}
\dot{\theta} = \zeta_2(t) h(\theta_t,\lambda)
\end{equation}

We additionally assume that:

\paragraph{(A5)} (\ref{eq:rcmdp_theta_ode_fixed}) has a globally asymptotically stable equilibrium $x(\lambda)$ such that $x$ is a Lipschitz map.

Assumption (A5) turns (\ref{eq:rcmdp_lambda_ode}) into:

\begin{equation} \label{eq:rcmdp_lambda_ode_fixed}
\dot{\lambda}(t) = g(x(\lambda_t),\lambda_t)
\end{equation}

Let's further assume that:

\paragraph{(A6)} The ODE (\ref{eq:rcmdp_lambda_ode_fixed}) has a globally asymptotically stable equilibrium $\lambda\opt$.

\paragraph{(A7)} $\sup_t ( \lVert \theta_t \rVert + \lVert \lambda_t \rVert ) < \infty$ almost surely.

\paragraph{Proof of Theorem \ref{th:rcmdp_pg_conv}}
\begin{proof}
	Above are the necessary conditions to apply Theorem 2 from chapter 6 of \cite{Borkar2009}, which shows that $(\theta_t,\lambda_t) \rightarrow (x(\lambda\opt),\lambda\opt)$. Now the saddle point theorem assures that $\theta\opt = x(\lambda\opt)$ maximizes the Lagrange optimization problem stated in (\ref{eq:rcmdp_lo_customized_obj}).
\end{proof}

\section{Reward Shaping in RCMDPs} \label{sec:apx_reward_shaping}
\subsection{Proof of \cref{th:reward_shaping_policy}} \label{subsec:apx_reward_shaping_policy_invariance_proof}
    \begin{proof}
    The robust optimal $q$-function satisfy the robust Bellman equation for the original RCMDP $\rcmdp$:
    \begin{equation*}
    \begin{aligned}
        \hat{q}^\star_\rcmdp(s,a) &= \min_{p\in\ambset_{s,a}} \E_{s'\sim p}\Big[ r'_{s,a,s'} + \gamma \max_{a'\in\actions}\hat{q}^\star_\rcmdp(s',a') \Big]
    \end{aligned}  
    \end{equation*}
    Subtracting $\lyapunov(s)$ and some algebraic manipulation gives:
    \begin{equation*}
    \begin{aligned}
        \hat{q}^\star_\rcmdp(s,a) + \lyapunov(s) &= \min_{p\in\ambset_{s,a}} \E_{s'\sim p}\Big[ r'_{s,a,s'} - \gamma \lyapunov(s') + \lyapunov(s) + \gamma \max_{a'\in\actions}\big(\hat{q}^\star_\rcmdp(s',a') + \lyapunov(s') \big) \Big]\\
        &\stackrel{(a)}{=} \min_{p\in\ambset_{s,a}} \E_{s'\sim p}\Big[ r'_{s,a,s'} - \big(\lyapunov(s') - \lyapunov(s)\big) + \max_{a'\in\actions}\big(\hat{q}^\star_\rcmdp(s',a') + \lyapunov(s') \big) \Big]\\
    \end{aligned}  
    \end{equation*}
    Here $(a)$ follows by setting $\gamma=1$, and considering the finite-horizon setting, i.e., $g^\pi(\xi, r) =  \sum_{t=0}^{N} \gamma^t r_{s_t,a_t,s_{t+1}}$. 
    
    We now define $\hat{q}_{\rcmdp'}(s,a) \stackrel{\delta}{=} \hat{q}^\star_\rcmdp(s,a) + \lyapunov(s)$ and set $f_{s,a,s'}= -\big(\lyapunov(s') - \lyapunov(s)\big)$. We therefore have:
    \begin{equation*}
    \begin{aligned}
        \hat{q}_{\rcmdp'}(s,a) &=
        \min_{p\in\ambset_{s,a}} \E_{s'\sim p}\Big[ \big(r'_{s,a,s'} + f_{s,a,s'} \big) + \max_{a'\in\actions}\hat{q}_{\rcmdp'}(s',a') \Big]\\
    \end{aligned}  
    \end{equation*}
    But this is exactly the Bellman equation for reward transformed RCMDP $\rcmdp'$. We then have: \[\hat{q}^\star_{\rcmdp'}(s,a) = \hat{q}_{\rcmdp'}(s,a) = \hat{q}^\star_{\rcmdp}(s,a)+\lyapunov(s)\]
    And the optimal policy for $\rcmdp'$ satisfies:
    \begin{equation*}
        \begin{aligned}
            \pi^\star_{\rcmdp'}(s) &\in \argmax_{a\in\actions}\hat{q}^\star_{\rcmdp'}(s,a)\\
            &= \argmax_{a\in\actions}\hat{q}^\star_{\rcmdp}(s,a) + \lyapunov(s)\\
            &= \argmax_{a\in\actions}\hat{q}^\star_{\rcmdp}(s,a)
        \end{aligned}
    \end{equation*}
    And is optimal for the original RCMDP $\rcmdp$ as well. Similarly, it can be shown that every optimal policy of original RCMDP $\rcmdp$ is also optimal for the transformed RCMDP $\rcmdp'$ simply by following exactly same steps as shown above, but with shaping function $-f_{s,a,s'}$ and the role of $\rcmdp$ and $\rcmdp'$ interchanged.
    
    Next, consider the case of infinite-horizon,i.e., $\lambda<1$. To extend the convergence result obtained for the case of finite-horizon with $\lambda=1$ to this case, we rely on the results in \cite{Altman2004}. Indeed, under the reasonable\footnote{Transient MDPs assume that the expected time we spend (under policy $\pi$) in any state $s$ is finite.} assumption of transient MDPs (Def. 7.1, p. 75, \cite{Altman2004}), we can conclude, in our specific case of finite-state and finite-action MDPs, that our MDPs are contracting (using the argument in \cite{Altman2004}, p. 99). Next, using Theorem 7.5, \cite{Altman2004}, we conclude that our MDPs admit a uniform Lyapunov function (in the sense of Def. 7.4, p. 77, \cite{Altman2004}). Finally, under the Slater feasibility condition, i.e., inequality (1b) satisfied, and using Theorem 15.5, p. 201, \cite{Altman2004}, we conclude that the value of the infinite-horizon problem converges to the value of the
finite-horizon one.
    \end{proof}

\end{document}